%% file: main.tex
\documentclass[runningheads]{llncs}

\usepackage{eccv}

\usepackage{eccvabbrv}

\usepackage{graphicx}
\usepackage{booktabs}
\usepackage{wrapfig}
\usepackage{tabularx,booktabs}
\usepackage{multirow}
\usepackage{microtype}

\usepackage[accsupp]{axessibility}  

\usepackage[symbol]{footmisc}

\usepackage{graphicx}
\usepackage{caption}

\usepackage{hyperref}

\usepackage{orcidlink}

\newcolumntype{Y}{>{\centering\arraybackslash}X}
\begin{document}

\title{Leveraging Imperfect Restoration for Data Availability Attack} 


\author{Yi Huang\inst{1}$^*$\orcidlink{0000-0002-4920-4333} \and 
Jeremy Styborski\inst{1}$^*$\orcidlink{0009-0005-3465-6213} \and
Mingzhi Lyu\inst{2}$^*$\orcidlink{0009-0005-9359-5908} \and
Fan Wang\inst{2}\orcidlink{0000-0001-8582-1673} \and
Adams Kong\inst{1}\orcidlink{0000-0002-9728-9511}}

\authorrunning{Y.Huang et al.}

\institute{College of Computing and Data Science, Nanyang Technological University, Singapore \and
Rapid-Rich Object Search (ROSE) Lab, Interdisciplinary Graduate Programme, Nanyang Technological University, Singapore \\
\email{shellbyhuang@gmail.com, \{styb0001,lyum0002,fan005,AdamsKong\}@ntu.edu.sg}}

\maketitle

\def\thefootnote{*}\footnotetext{Equal contribution}
\renewcommand*{\thefootnote}{\arabic{footnote}}


\begin{abstract}

The abundance of online data is at risk of unauthorized usage in training deep learning models. To counter this, various Data Availability Attacks (DAAs) have been devised to make data unlearnable for such models by subtly perturbing the training data. However, existing attacks often excel against either Supervised Learning (SL) or Self-Supervised Learning (SSL) scenarios. Among these, a model-free approach that generates a Convolution-based Unlearnable Dataset (CUDA) stands out as the most robust DAA across both SSL and SL. Nonetheless, CUDA's effectiveness against SSL is underwhelming and it faces a severe trade-off between image quality and its poisoning effect. In this paper, we conduct a theoretical analysis of CUDA, uncovering the sub-optimal gradients it introduces and elucidating the strategy it employs to induce class-wise bias for data poisoning. Building on this, we propose a novel poisoning method named Imperfect Restoration Poisoning (IRP), aiming to preserve high image quality while achieving strong poisoning effects. Through extensive comparisons of IRP with eight baselines across SL and SSL, coupled with evaluations alongside five representative defense methods, we showcase the superiority of IRP. \textbf{Code:} \ \url{https://github.com/lyumingzhi/IRP}

\end{abstract}


\section{Introduction}
\label{sec:intro}

The proliferation of online data has proven an indispensable resource for the advancement of deep learning models. However, the collection of certain datasets without explicit consent poses a possible threat to personal privacy \cite{birhane2021large}. Additionally, the emergence of generative AI presents new challenges, as training with unauthorized data could infringe upon the owners' copyright \cite{verge_article}.

DAAs have emerged as a promising strategy to address the issue of unauthorized data usage\cite{feng2019learning, yuan2021neural, huang2021unlearnable, fowl2021adversarial, tao2021better, sandoval2022autoregressive, yu2022availability, wu2022one, he2023indiscriminate, sadasivan2023cuda, fu2022robust}. These attacks introduce subtle perturbations into training data, hindering the model's ability to effectively learn useful information, which subsequently leads to poor performance on unseen data. This poor performance manifests as a substantial decrease in accuracy when the model is tested on clean data in classification tasks. Currently, the majority of DAAs are developed within the framework of SL. However, a recent study by He et al. \cite{he2023indiscriminate} reveals that Adversarial Poison (AP) \cite{fowl2021adversarial}, a cutting-edge DAA targeting SL, lacks effectiveness when applied to SSL. In response, they introduce a novel approach, termed Contrastive Poison (CP) \cite{he2023indiscriminate}, to counter SSL. Their research leads to several questions: (1) Are other DAAs designed for SL also ineffective against SSL (2) Does the proposed CP method maintain similar effectiveness in SL scenarios? (3) Can we develop a DAA capable of conducting effective attacks in both SL and SSL scenarios?

\begin{wrapfigure}[21]{r}{0.5\textwidth}
    \vspace{-1.1cm}
    \begin{center}
        \includegraphics[width=\linewidth]{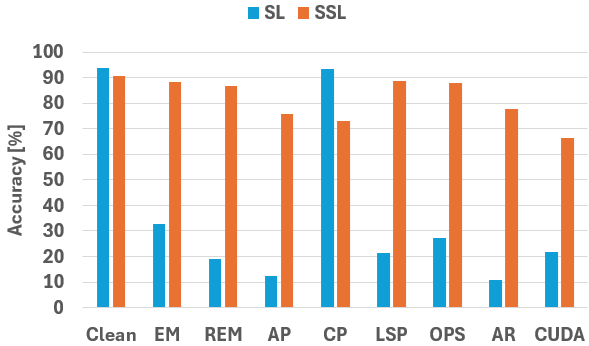}
    \end{center}
    \vspace{-0.5cm}
    \caption{The clean test accuracies of different DAAs on CIFAR-10 \cite{krizhevsky2009learning}. Lower accuracy represents better attack performance.}
    \label{fig_effdrop}
    \vspace{-0.1cm}
    \begin{center}
        \includegraphics[width=0.5\textwidth]{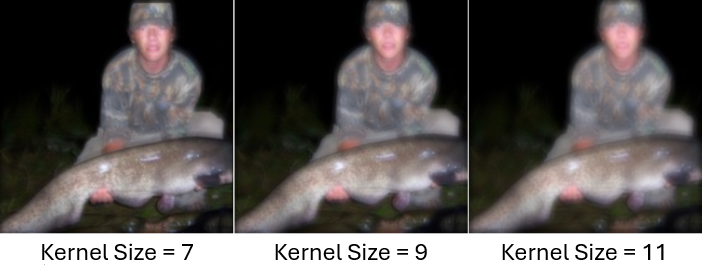}
    \end{center}
    \vspace{-0.5cm}
    \caption{The poisoned images generated by CUDA with different kernel sizes.}
    \label{kernel_quatliy}
\end{wrapfigure}

To answer these questions, we experimentally evaluate seven representative DAAs designed for SL on SSL. SSL is a strong defense against most DAAs due to its augmentation invariance objective; DAA features that are distorted or obfuscated by augmentations are ignored by the SSL algorithm and cannot affect downstream tasks. Accordingly, we find in Fig. \ref{fig_effdrop} that most poisons for CIFAR-10 \cite{krizhevsky2009learning} fail to generalize to SSL. Only CUDA \cite{sadasivan2023cuda} and CP \cite{he2023indiscriminate} exhibit significant impact on SSL. CUDA filters inject poison features throughout images that are resistant to cropping, flipping, and color shifts, and CP is designed specifically to counter SSL. We note a significant loss of performance for CP in SL scenarios. Therefore, CUDA emerges as the most potent DAA in both SSL and SL. However, the clean test accuracy of SSL trained on CUDA still hovers around 70\%, suggesting that models trained on poisoned data retain some usability. As suggested by Sadasivan et al. \cite{sadasivan2023cuda}, increasing the size of the convolution kernels in CUDA could provide a stronger poisoning effect, but it also degrades image quality. Fig. \ref{kernel_quatliy} shows that increased kernel size results in obviously blurred images. This compromise is often unacceptable in practical scenarios, particularly for sharing images and artwork online.

To comprehend CUDA's efficacy, Sadasivan et al. provide a theoretical analysis under assumptions of Gaussian-distributed data with independent elements and a two-class scenario. However, these assumptions may not hold in real-world scenarios. In addition, their analysis is not based on deep learning, making it challenging to derive insights for designing a more effective poison method. To gain a deeper understanding of CUDA's impact on model training, we first conduct a theoretical analysis of CUDA using a deep learning model. Our analysis reveals how CUDA generates sub-optimal gradients for clean data and introduces class-wise bias through random filters. Building upon this analysis, we propose a new poisoning method, named Imperfect Restoration Poisoning (IRP), which maintains high image quality while achieving a stronger poisoning effect. We compare IRP to eight representative DAAs in both SL and SSL scenarios. In the SL scenario, we additionally compare IRP with other DAAs under various defense mechanisms, including adversarial training \cite{tao2021better}, Image Shortcut Squeezing \cite{liu2023image}, Mixup \cite{zhang2018mixup}, Cutout \cite{devries2017improved}, and CutMix \cite{yun2019cutmix}. All experimental results indicate that IRP significantly outperforms previous DAAs. Our main contributions are:
\begin{itemize}
        \item We experimentally assess eight representative DAAs in both supervised and self-supervised learning scenarios, revealing that existing DAAs fail to achieve satisfactory effectiveness simultaneously in both contexts.
        \item We conduct a theoretical analysis of CUDA on deep learning models, revealing how it generates sub-optimal gradients for clean data and introduces class-wise bias through random filters.
        \item We propose a novel DAA named IRP based on imperfect restoration, which achieves high effectiveness in both SL and SSL and maintains superior image quality compared to CUDA.
        \item We perform an exhaustive comparison between IRP and eight other DAAs across SL, SSL, and five defense techniques on CIFAR-10 and a subset of ImageNet \cite{deng2009imagenet}. IRP outperforms baseline poisons by a large margin in all cases. We further verify IRP across five additional architectures and two supplementary datasets. IRP displays high effectiveness across all experiments.
\end{itemize}


\section{Related Works}

DAAs have garnered significant attention in recent years for their application to data protection. Various methodologies have been proposed to deliberately induce malfunctions in the target model in order to safeguard against unauthorized data usage. DAAs can be broadly categorized into two types, model-reliant methods \cite{feng2019learning, yuan2021neural, huang2021unlearnable, fowl2021adversarial, fu2022robust, he2023indiscriminate} and model-free methods \cite{sandoval2022autoregressive, wu2022one, yu2022availability, sadasivan2023cuda}.

\noindent\textbf{Model-Reliant Methods:}
Early approaches to DAAs often frame the poisoning problem as a bilevel optimization task, balancing loss minimization with respect to model parameters with loss maximization with respect to perturbed inputs. While these methods \cite{jagielski2018manipulating, biggio2012poisoning} initially showed promise, their applicability to deep neural networks is limited due to intractability in obtaining exact solutions. Efforts to address optimization challenges in perturbation generation, as seen in works by Feng et al. \cite{feng2019learning} and Yuan et al. \cite{yuan2021neural}, often introduce significant computational overhead, limiting scalability for real-world applications. Alternatively, Huang et al.\cite{huang2021unlearnable} propose another bi-level Error Minimizing (EM) approach, deliberately optimizing perturbations to minimize training loss. By alternately training a surrogate model and optimizing perturbations, the efficacy of the perturbations is attained through multiple rounds of optimization. In contrast to traditional bi-level approaches, Fowl et al. \cite{fowl2021adversarial} demonstrate the effectiveness of utilizing common objectives from adversarial examples for generating potent Adversarial Poisons (AP). However, Tao et al. \cite{tao2021better} pinpoint that these DAAs can be easily broken by adversarial training. In response, Fu et al. \cite{fu2022robust} introduce Robust Error-Minimizing (REM) which enhances the poisoning effects against adversarial training by replacing the surrogate in EM with an adversarially-trained model. Recently, He et al. \cite{he2023indiscriminate} assert that all aforementioned methods are tailored for supervised scenarios, thereby making them unsuitable for self-supervised learning. To address this gap, they introduce Contrastive Poison (CP), specifically designed for self-supervised learning. However, their assessment of DAAs tailored for SL is exclusively centered on AP, with other DAAs left unexplored. Additionally, while they show some effectiveness of their class-wise CP on both SL and SSL, the test accuracies on SSL remain high, ranging from 60.7\% to 68.0\%. In addition, class-wise additive perturbations can be recovered by the average image of a class, making them relatively easy to remove \cite{sandoval2022autoregressive}. Therefore, we do not consider the class-wise poisons that apply the same additive noise to images, including class-wise attacks by EM, AP, and CP.

\noindent\textbf{Model-Free Methods:}
Model-free DAAs tackle the data protection issue by leveraging shortcut learning. These methods are grounded in the observation that Deep Neural Networks (DNNs) often prioritize easily learnable shortcuts over semantic features, allowing them to distinguish between examples from different classes \cite{caron2020finite, ilyas2019adversarial}. Expanding on this concept, Yu et al. \cite{yu2022availability} propose incorporating artificially generated Linearly Separable Patterns (LSP) sampled from high-dimensional Gaussian distributions into images as an effective method for data poisoning. Sandoval-Segura et al. \cite{sandoval2022autoregressive} discover that applying additive perturbations derived from autoregressive (AR) processes to clean data can serve as an effective shortcut. Wu et al. \cite{wu2022one} examine the model's susceptibility to sparse poisons and demonstrate that consistently perturbing only one pixel is adequate to generate potent poisons. In contrast to applying additive perturbations, Sadasivan et al. \cite{sadasivan2023cuda} introduce a Convolution-based Unlearnable Dataset (CUDA) generation method, which utilizes randomly generated class-wise convolutional filters. In our evaluation, CUDA demonstrates moderate robustness to SSL. However, its effectiveness on SSL remains insufficient and CUDA grapples with a severe trade-off between image quality and poison efficacy.

\noindent\textbf{DAA Mitigation Strategies:} 
Previous studies have demonstrated that Adversarial Training (AT) \cite{goodfellow2014explaining} is an effective method to mitigate the potency of DAAs \cite{tao2021better,wen2023is}. However, AT is impeded by its significant computational demands. Alternatively, various image preprocessing techniques, such as data augmentation (e.g., Mixup \cite{zhang2018mixup}, Cutout \cite{devries2017improved}, and CutMix \cite{yun2019cutmix}), have been explored. While these augmentation-based techniques show some impact, they are not as effective as AT \cite{fowl2021adversarial}. Recently, Liu et al. \cite{liu2023image} introduce Image Shortcut Squeezing (ISS) \cite{liu2023image}, a compression-based approach to counter DAAs. They demonstrate that ISS surpasses previously studied data augmentation countermeasures and achieves results comparable to AT. We evaluate the proposed IRP against multiple countermeasures, including AT, ISS, Mixup, Cutout, and CutMix.


\section{Analysis}

\subsection{Threat Model}

Like CUDA, most DAA papers define an ``attacker'' that applies a DAA on a labeled dataset that they want to protect yet publish. The attacker assumes no knowledge of any downstream DNN training process. It is therefore crucial to demonstrate that DAAs are effective across multiple training methods, including SL, SSL, and AT. The attacker generates class-based perturbations to their image set in order to inject poison features. The attacker publishes their poisoned images, optionally with captions or labels. Model trainers collect these poisoned images into datasets and may further alter or label the images.

\subsection{Notation and Preliminaries}

For clarity, we define a set of notations. $f$ represents a Convolution Neural Network (CNN), including its training loss function but excluding the first layer convolution 2D filter $W \in \mathbb{R}^{d_1 \times d_2}$. $X \in \mathbb{R}^{m_1 \times m_2}$ denotes an input image and $X_c$ represents an image belonging to class $c \in [1, ..., \mathbb{C}]$, where $\mathbb{C}$ is total number of the classes. Given that cross-correlation is predominantly used in place of convolutions in practice, the first layer operation is denoted as $X \star W$, where $\star$ signifies cross-correlation. Note that $W$ is a part of the CNN and $f(X \star W)$ is the loss value of $X$. To retain readability, we focus on gradient analysis of the first layer filter, $W$. CUDA employs class-specific convolutional filters for data poisoning, and we provide a brief overview of its operation. Let $R_c \in \mathbb{R}^{\kappa \times \kappa}$ be a CUDA filter for class $c$ images. According to \cite{sadasivan2023cuda}, the generation procedure for $R_c$ involves setting one random parameter among the $\kappa^2$ parameters to 1 within each $R_c$. Concurrently, the remaining parameters are initialized randomly, drawn from a uniform distribution with support $[0, p_u]$. An image $X_c$ belonging to class $c$ is poisoned by cross-correlation with the corresponding CUDA filter $R_c$, denoted as $X_c \star R_c$.

\subsection{The Non-Optimal Gradient}

To understand how CUDA lowers the accuracy of clean data, we analyze the gradient of $f$ for both clean and CUDA-poisoned data. Without loss of generality, the training set consists of $\mathbb{C}$ clean images, each belonging to a different class, with corresponding poisoned images denoted as $ X_i \star R_i $, $i \in [1, ..., \mathbb{C}]$. The gradient of the clean training data is
\begin{equation}
    \frac{\partial \sum_{i=1}^\mathbb{C} f(X_i \star W)}{\partial W} = \sum_{i=1}^\mathbb{C} X_i \star \frac{\partial f(u_i)}{\partial u_i},
    \label{ep1}
\end{equation}
where $u_i$ is an intermediate variable equal to  $X_i \star W$. The gradient of the poisoned training data is 
\begin{equation}
    \frac{\partial \sum_{i=1}^\mathbb{C} f(( X_i \star R_i)\star W)}{\partial W} = \sum_{i=1}^\mathbb{C} (X_i \star R_i) \star \frac{\partial f(u_{p, i})}{\partial u_{p, i}},
     \label{ep2}
\end{equation}
where the subscript $p$ indicates that the intermediate variable is from poisoned data. Using these gradients to update $W$, we obtain
\begin{equation}
\begin{split}
    W_{t+1} = W_t - \lambda \sum_{i=1}^\mathbb{C}  X_i \star \frac{\partial f(u_{i})}{\partial u_{i}}
    \\
    W_{p, t+1} = W_t - \lambda \sum_{i=1}^\mathbb{C} (X_i \star R_i ) \star \frac{\partial f(u_{p,i})}{\partial u_{p,i}}
    \end{split}
    \label{ep3}
\end{equation}
where $\lambda$ is learning rate. Inputting all clean training images, $X_1,\cdots, X_C$ to $f(\cdot \star W_{t+1})$ and $f(\cdot \star W_{p,t+1})$, we obtain clean and poison loss values $\sum_{i=1}^{\mathbb{C}}f \left( X_i \star W_{t+1} \right)$ and $\sum_{i=1}^{\mathbb{C}}f \left( X_i \star W_{p, t+1} \right)$, respectively. Given that $\sum_{i=1}^{\mathbb{C}} X_i \star \frac{\partial f(o_i)}{\partial i} $ in Eq. \ref{ep1} is the optimal gradient descent direction for the clean data, any deviation from this direction will result in increased loss when $\lambda$ is small. Thus, $\sum_{i=1}^{\mathbb{C}} f \left( X_i \star W_{p,t+1} \right) > \sum_{i=1}^{\mathbb{C}} f \left( X_i \star W_{t+1}  \right)$ for clean data. In other words, CUDA employs a non-optimal gradient to disrupt training, rendering clean data unrecognizable.

\vspace{-0.1cm}
\subsection{Poisoning through Class-Wise Bias}

Though the previous subsection uncovers how CUDA filters increase training loss, it does not indicate how these filters poison the network. To further investigate its mechanism, we input $X_c \star R_c$ to $f(\cdot \star W_{p,t+1} )$, whose first layer output is 
\begin{equation}
(X_c \star R_c) \star W_{p,t+1} = (X_c \star R_c) \star W_t - \lambda (X_c \star R_c) \star \sum_{i=1}^{C} \left((X_i \star R_i) \star \frac{\partial f(u_{p,i})}{\partial u_{p,i}}\right).
  \label{ep5}
    \end{equation}
Assuming $W_t$ is not poisoned by CUDA filters, we ignore the term $(X_c \star R_c)\star W_t$. Since cross-correlation is equivalent to a rotated convolution (i.e., $X_c \star R_c = X_c \ast R_{c\pi}$, where $\ast$ represents convolution and $R_{c\pi}$ represents rotating the 2D filter $R_c$ by $\pi$), the second term of Eq. \ref{ep5} can be rewritten as 
\begin{equation}
\lambda X_c * \sum_{i=1}^{C}  R_{c \pi} * R_i * X_{i \pi} * \frac{\partial f(u_{p,i})}{\partial u_{p,i}}.
 \label{ep6}
\end{equation}
Eq. \ref{ep6} shows that the output of the first layer is influenced by $R_{c \pi} * R_i$. Since $R_{c \pi} * R_i = R_{c \pi} \star R_{i \pi}$, we analyze $R_{c \pi} \star R_{i \pi}$. Theorem 1 describes the statistical probabilities of $R_{c \pi} \star R_{i \pi}$ and $R_{c \pi} \star R_{c \pi}$. The proof is provided in Appendix \ref{subsec:Proof_Thm1}.

\begin{theorem}
    Given two different filters, $R_i$ and $R_c$, generated by CUDA, if $1 > (\kappa^2 - 2) p_u^2 + 2p_u$, then
    $R_{c\pi}\star R_{i\pi}$ and $R_{c\pi}\star R_{c\pi}$ have the following properties:
    \begin{enumerate}
        \item The peak of $R_{c\pi}\star R_{c\pi}$is located at the center.
        \item The peak of $R_{c\pi}\star R_{i\pi}$ occurs at the position where two ones in  $R_{c\pi}$ and  $R_{i\pi}$ appear at the same position in $R_{c\pi}\star R_{i\pi}$.
        \item The probability that the peak of $R_{c\pi}\star R_{i\pi}$ lies in the center is $\frac{1}{\kappa^2}$.
        \item The expected peak of $R_{c\pi}\star R_{c\pi}$ is higher than the expected peak of $R_{c\pi}\star R_{i\pi}$.
    \end{enumerate}
\end{theorem}

Figure \ref{fig_ri} presents two CUDA filters along with the corresponding $R_{c \pi} \star R_{c \pi}$ and $R_{c \pi} \star R_{i \pi}$ to illustrate these statistical properties. In practice, CUDA utilizes $p_u=0.06$ and $\kappa=9$ for ImageNet, which fulfills $1 > (\kappa^2 - 2) p_u^2 + 2p_u$. Similar properties also apply on CIFAR-10 (see Appendix \ref{subsec:Proof_CIFAR10}). These properties indicate that $R_{c \pi} * R_c$ in Eq. \ref{ep6} retains more information in $X_{c \pi} * \frac{\partial f(u_{p,c}) } {\partial u_{p,c}}$ but $R_{c \pi} * R_i$ shifts the information in $X_{i \pi} * \frac{\partial f(u_{p,i})} {\partial u_{p,i}}$ to a different position with a probability of $(\kappa^2 - 1) / \kappa^2$. For $\kappa=9$, the probability is $0.988$. Consequently, the upper layers receive class $c$ information in the correct position, but information from other classes is shifted to incorrect positions. For $p_u=0.06$, we note the peak value of $R_{c \pi} \star R_{i\pi}$ is smaller than the sum of its non-peak values (estimated mean ratio of 0.099). This means that non-peak values also play an important role in degrading other class information. This blurring effect also affect the $c$-class, though to a lesser extent. More clearly, the estimated mean ratio of the peak value of $R_{c \pi} \star R_{c\pi}$ over the sum of its non-peak values is 0.106. Combining these two effects, CUDA filters utilize peak shifts and unique blur patterns to introduce class-wise bias in the poisoned class. 

\begin{figure}[h]
\vspace{-0.5cm}
  \begin{center}
    \includegraphics[width=\textwidth]{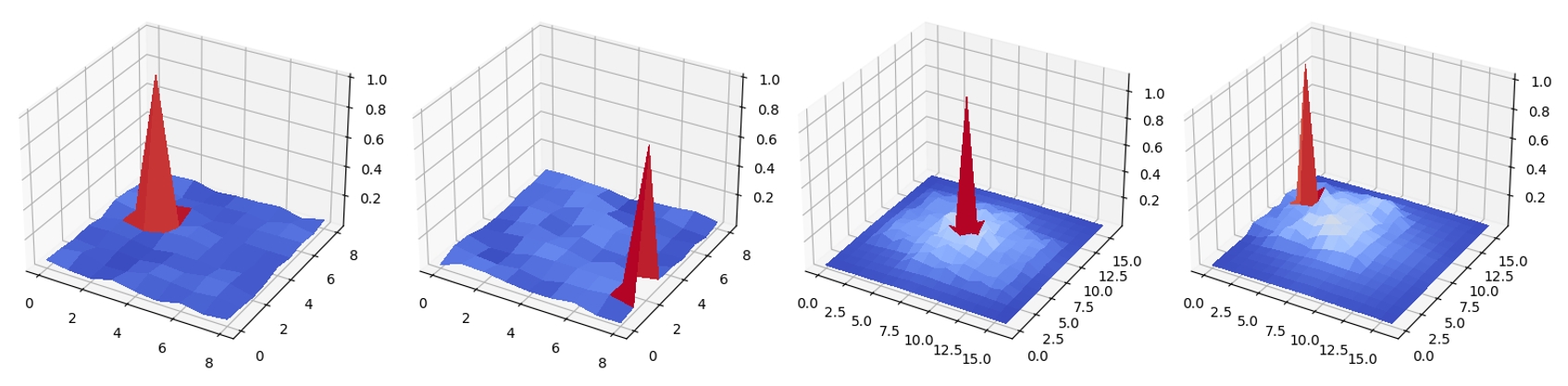}
  \end{center}
  \vspace{-0.2cm}
  \caption{Two CUDA filters $R_{c \pi}$ and $R_{i\pi}$(the first and second plots) along with the corresponding $R_{c \pi} \star R_{c \pi}$ (the third plot) and $R_{c \pi} \star R_{i \pi}$ (the last plot). }
  \label{fig_ri}
  \vspace{-1cm}
\end{figure}

\subsection{Poisoning with Imperfect Restoration}

The previous analysis shows that CUDA filters create class-wise bias, and that $R_{c\pi} \star R_{c\pi}$ retains the training information of class $c$. When $X_c \star R_c$ is closer to $X_c$, it is expected to preserve the class $c$ information better and enhance the image quality. However, if $X_c \star R_c = X_c$, then there is no poisoning effect. Thus, we propose the Imperfect Restoration Poison (IRP). Let $X_{ci}$ be one of the class $c$ images, where $i\in{1,\cdots,n}$. IRP first applies CUDA filter $R_c$ on all class $c$ training images and obtains $X_{ci} \star R_c$. Then, the filtered images are cropped into patches with the size same as the CUDA filter i.e., $\kappa \times \kappa$, denoted as $(X_{ci} \star R_c)_j$. The pixel value of $X_{ci}$ corresponding to the center of the patch is denoted $y_{cij}$, and the patch is reshaped into a column vector denoted as $\eta_{cij}$. Finally, a class-wise linear model $\alpha_c \in \mathbb{R}^{\kappa^2 \times 1}$ and the mean square loss are used to estimate the value of $y_{cij}$. More precisely, the class-wise model $\widehat{\alpha_c}$ is obtained by minimizing 
\begin{equation}
    \widehat{\alpha_c}= \underset{\alpha_c}{\text{argmin}} \sum_i \sum_j \| y_{cij} - \alpha_c^T \eta_{cij} \|_2^2
\end{equation}
Once $\widehat{\alpha_c}$ is obtained, it is reshaped back to a $\kappa \times \kappa$ filter, denoted as $A_c$. Finally, IRP uses $X_{ci} \star P_c$ to perform poisoning, where $P_c = R_c \star A_c$. Theorem \ref{theorem2} establishes that IRP is incapable of achieving perfect restoration (proof in Appendix \ref{subsec:Proof_Thm2}).
\begin{theorem}
\label{theorem2}
Given a CUDA filter size larger than 1 and $x_{cij}$ as the vector form of the patch in $X_{ci}$ involved in producing $\eta_{cij}$ (i.e., $\eta_{cij} = \wp x_{cij}$, where $\wp$ is a $\kappa^2 \times (3\kappa-2)^2$ matrix), if the column vectors of $\wp$, excluding the one corresponding to $y_{cij}$, span $\mathbb{R}^{\kappa^2}$, then, $\exists x_{cij} \in \mathbb{R}^{(3\kappa-2)^2}$ such that $y_{cij} \neq a_c^T \wp x_{cij}$.
\end{theorem}

\begin{figure}[h]
\vspace{-0.4cm}
  \begin{center}
    \includegraphics[width=\textwidth]{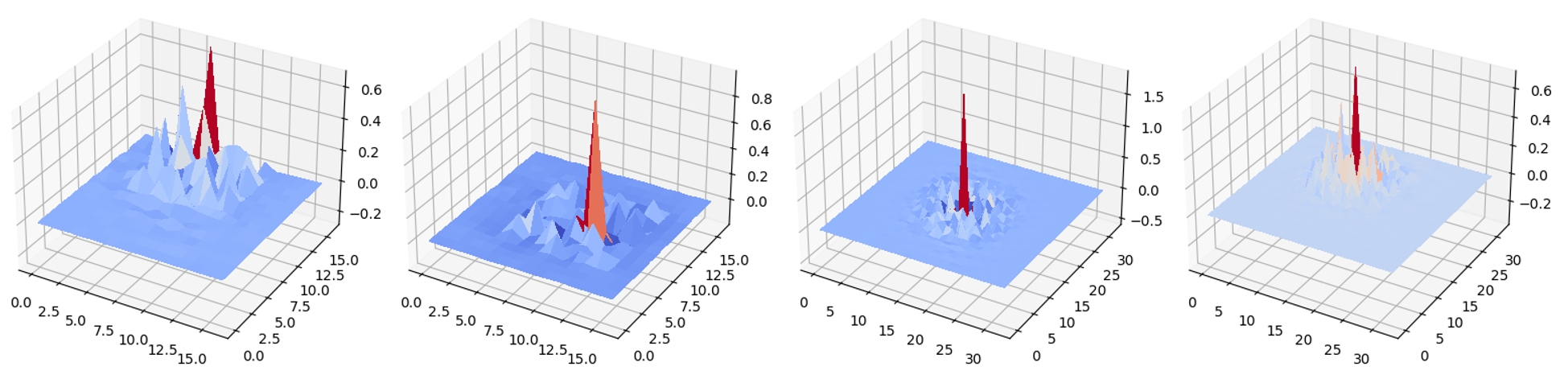}
  \end{center}
  \caption{Two IRP filters $P_{c\pi}$ and $P_{i\pi}$ (the first and second plots) along with the corresponding $P_{c\pi} \star P_{c\pi}$ (the third plot) and $P_{c\pi} \star P_{i\pi}$ (the last plot).}
  \label{fig_ari}
\end{figure}

Fig. \ref{fig_ari} shows two IRP filters with the corresponding self-correlation and cross-correlation. Similar to CUDA, we observe that $P_{c\pi} \star P_{c\pi}$ exhibits a peak at the center, while $P_{c\pi} \star P_{i\pi}$ has a peak in a different position. However, they have more complex patterns around the peaks. Therefore, like CUDA, IRP exploits peak shifts to create class-wise bias. However, IRP filters have turbulent patterns around the peaks, creating unique high-frequency attack patterns that are more complex than the low-frequency patterns from CUDA. Due to their dependence on the training data, it would be challenging to theoretically analyze IRP filters, as was done for CUDA, without proper assumptions on the data distribution. We refrain from imposing unrealistic assumptions on the data distribution (e.g., Gaussian) as this could lead to misleading analytic results. 

Fig. \ref{cuda_irp} displays some images generated by IRP and CUDA. IRP images are visually clearer and maintain higher quality. Notably, the man's face is sharper and the spots of the salamander and stingray are easily distinguishable under IRP protection. The backgrounds of all three images are sharper under IRP.

\begin{figure}[h]
  \centering
  \begin{subfigure}[b]{0.495\textwidth}
    \includegraphics[width=\textwidth]{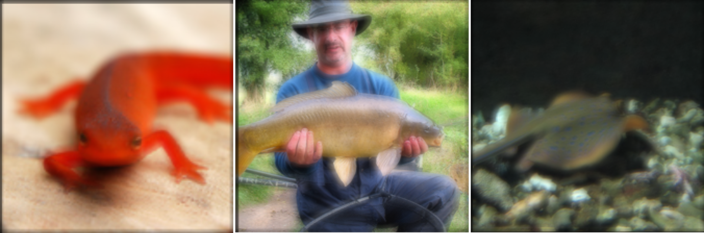}
    \caption{CUDA}
    \label{cuda_fig}
  \end{subfigure}
  \begin{subfigure}[b]{0.495\textwidth}
    \includegraphics[width=\textwidth]{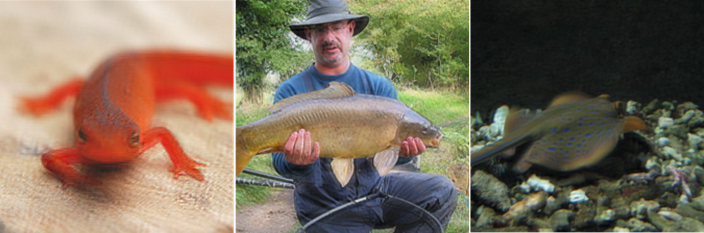}
    \caption{IRP}
    \label{irp_fig}
  \end{subfigure}
  \caption{Poisoned examples generated by CUDA (left) and IRP (right).}
  \label{cuda_irp}
\vspace{-0.4cm}
\end{figure}


\section{Experiments}

In this section, we begin by evaluating the effectiveness of IRP alongside eight DAA methods in standard SL and SSL settings, as well as against representative defense methods. We then investigate the effectiveness of different DAAs under partial poisoning scenarios. For model-reliant methods, our experiments primarily focus on the nominal setup, wherein a surrogate model generates poison images. To assess the transferability of IRP, we evaluate IRP across six different architectures on three datasets. Finally, we compare the image quality generated by CUDA and IRP and ablate different filter types across various training methods.

\subsection{Experimental Setup}

\textbf{Datasets:}
We examine poisoning on CIFAR-10 \cite{krizhevsky2009learning}, CIFAR-100 \cite{krizhevsky2009learning}, STL-10 \cite{coates2011analysis}, and ImageNet-100, a subset of ImageNet \cite{deng2009imagenet} consisting of 100 classes. By default, we use 50,000 images for training and 10,000 images for testing on CIFAR-10 and CIFAR-100. For STL-10, we use their default 5000 training images and 8000 test images for training and testing, respectively. For ImageNet-100, we train on the images from the first 100 classes of the official training set and test on all corresponding images from the official validation set.

\par\noindent\textbf{Poison Settings:} 
We compare IRP with eight state-of-the-art DAAs, including four model-reliant methods, including EM \cite{huang2021unlearnable}, TAP \cite{fowl2021adversarial}, REM \cite{fu2022robust}, and CP \cite{he2023indiscriminate}, and four model-free methods, including CUDA \cite{sadasivan2023cuda}, LSP \cite{yu2022availability}, OPS \cite{wu2022one}, and AR \cite{sandoval2022autoregressive}.
As generating poison data using model-reliant methods is time-consuming on ImageNet, we generate a 20\% subset of ImageNet-100 for CP, TAP, and EM. We maintain all default settings and retain the $L_{\infty}$ norm of the defensive perturbation at 8/255. We utilize the full REM-poisoned ImageNet-100 dataset shared by the REM authors. According to \cite{fu2022robust}, the $L_{\infty}$ norm of the defensive perturbation $\rho_u$ is set to $8/255$, and the adversarial perturbation radius $\rho_a$ is set to $4/255$ for controlling the level of protection of the noise against adversarial training. For model-free methods, we maintain the settings as per their original papers, except for LSP and OPS on ImageNet-100. As the original settings could not yield a satisfactory poisoning effect, we amplify the perturbation of LSP and OPS to $L_{\infty}=16/255$ and $L_0=5$, respectively. As AR does not report its effectiveness on ImageNet-100 and has shown to be vulnerable to SSL and AT on CIFAR-10, we opt not to include AR in further testing on ImageNet-100.

\textbf{Models:}
Unless stated otherwise, we employ ResNet-18 \cite{he2016deep} as the surrogate for model-reliant methods and as the target models for all the testing methods in both SL and SSL scenarios. Various SSL algorithms are employed, including SimCLR \cite{chen2020simple}, SimSiam\cite{chen2021exploring}, MoCoV3\cite{chen2021mocov3} and BYOL\cite{grill2020bootstrap}. Linear probing accuracy, consistent with CP \cite{he2023indiscriminate}, is utilized to evaluate the effectiveness of DAAs in SSL. To assess transferability, we evaluate IRP on target models with various architectures, including ResNet-18, ResNet-34 \cite{he2016deep}, VGG-19 \cite{simonyan2014very}, DenseNet-121 \cite{huang2017densely}, MobileNet-V2 \cite{sandler2018mobilenetv2}, and ViT \cite{dosovitskiy2020image}. Training details are deferred to Appendix \ref{subsec:Training_Details}.

\subsection{Comparisons to Baseline DAAs}
\label{baslines}

\begin{table}[t]
 \caption{Top-1 clean test accuracies (\%) of ResNet-18 trained with data poisoned by various DAAs on CIFAR-10. The ``Max'' column represents the highest clean test accuracy, which corresponds to the worst performance of each DAA under different training scenarios. Bold indicates the best performance, and underline denotes the second-best performance.}
\label{tab:compare_cifar}
\centering
\begin{tabularx}{0.9\textwidth}{l| *{2}{Y}|*{5}{Y}|c}
\toprule
DAA  & SL & SSL & AT & Cutout & Cutmix & Mixup & ISS & \multicolumn{1}{c}{Max} \\
  \midrule
Clean & 93.86 & 90.55 & 89.57 & 95.62 & 95.78 & 95.46 & 82.40 & 95.62 \\
\midrule
EM & 32.53 & 88.45 & 89.31 & 36.18 & 38.40 & 54.19 & 82.12 & 89.31 \\
REM & 18.91 & 86.70 & 34.45 & 19.67 & 27.46 & 23.27 & 80.04 & 86.70\\
AP & 12.45 & 75.77 & 85.67 & \textbf{9.70} & \textbf{8.38} & 33.27 & 81.10 & 85.67 \\
CP & 93.59 & 73.08 & 88.43 & 94.51 & 93.77 & 93.36 & 82.21 & 94.51 \\
LSP & 21.47 & 88.87 & 85.67 & 18.47 & 21.68 & 22.76 & 79.33 & 88.87  \\
OPS & 27.24 & 87.79 & \textbf{20.10} & 65.12 & 85.27 & 37.56 & 72.14 & 87.79\\
AR & \underline{10.95} & 86.86 & 72.94 & 12.17 & 12.95 & \textbf{14.15} & 82.83 & 86.86 \\
CUDA & 21.89 & \underline{66.58} & 48.58 & 23.46 & 24.04 & 21.75 & \textbf{21.94} & \underline{66.58}\\
IRP & \textbf{10.39} & \textbf{43.24} & \underline{32.21} & \underline{10.85} & \underline{15.21} & \underline{16.03} & \underline{29.42} & \textbf{43.24 } \\
\bottomrule
\end{tabularx}
\end{table}

\textbf{Effectiveness on SL and SSL:}
We first evaluate IRP and the other DAAs on standard SL and SSL training scenarios on CIFAR-10 and ImageNet-100. SimCLR is utilized for model pretraining in the SSL evaluation. Columns 2-3 of Tables \ref{tab:compare_cifar} and \ref{tab:compare_IN} show that IRP achieves state-of-the-art poisoning, enforcing the lowest clean test accuracies in both SL and SSL and for both CIFAR-10 and ImageNet-100.

We begin by analyzing results for CIFAR-10 in Table \ref{tab:compare_cifar}. For SL with CIFAR-10, AR and IRP demonstrate a similar level of poisoning effect under SL, reducing the test accuracy to around 10\%. However, AR underperforms on SSL, permitting a significantly higher accuracy of 86.86\% compared to that of IRP at 43.24\%. CUDA is the second-best DAA for SSL with an accuracy of 66.58\%, which is still drastically higher than that of IRP. 

For ImageNet-100 results in Table \ref{tab:compare_IN}, we begin by nothing the drop in effectiveness of OPS relative to CIFAR-10. This is attributed to the RandomResizeCrop augmentation utilized for training models on ImageNet-100. We chose to use this augmentation on ImageNet-100 as it is commonly used for SL and SSL on ImageNet-100 in literature and it enhances the effectiveness of all the defense methods. For SL with ImageNet-100, CUDA, EM, AP, and LSP demonstrate strong poisoning capabilities, reducing the clean test accuracy to below 10\%. Even so, IRP outperforms all others with a test accuracy of 1.98\%. For SSL with ImageNet-100, the test accuracies of most DAAs remain high, except for CUDA at 21.89\% and CP at 18.76\%. IRP outperforms both CUDA and CP, achieving a clean test accuracy of 9.30\%. These results establish IRP as the most effective DAA under both SL and SSL scenarios.

\begin{table}[t]
 \caption{Top-1 clean test accuracies (\%) of ResNet-18 trained with data poisoned by various DAAs on ImageNet-100. The ``Max'' column represents the highest clean test accuracy, which corresponds to the worst performance of each DAA under different training scenarios. Bold indicates the best performance, and underline denotes the second-best performance.}
   \label{tab:compare_IN}
  \centering
\begin{tabularx}{0.95\textwidth}{l| *{2}{Y}|*{5}{Y}|c}
\toprule
DAA & SL & SSL & AT & Cutout & Cutmix & Mixup & ISS & \multicolumn{1}{c}{Max} \\
\midrule
Clean & 77.66 & 70.96 & 69.76 & 78.04 & 81.08 & 80.38 & 71.58 & 80.38 \\
\midrule
EM & 6.32 & 50.02 & 43.26 & 6.42 & 5.80 & 12.38 & 41.18 & 50.02 \\
REM & 14.18 & 65.30 & 57.34 & 15.60 & 16.10 & 33.08 & 67.52 & 67.52 \\
AP & 8.18 & 41.80 & 42.88 & 7.68 & 8.38 & 9.68 & 24.22 & 42.88 \\
CP & 57.46 &\underline{18.76} & 48.76 & 58.20 & 62.28 & 61.76 & 50.18 & 62.28 \\
LSP & 6.62 & 62.06 & \underline{28.92} & \underline{5.06} & 7.84 & 5.50 & 33.70 & 62.06 \\
OPS & 51.00 & 62.22 & 52.56 & 53.86 & 65.12 & 48.82 & 57.36 & 65.12 \\
CUDA & \underline{6.10} & 26.12 & 36.34 & 8.66 & \underline{5.70} & \underline{5.16} & \textbf{3.58} & \underline{36.34} \\
IRP & \textbf{1.98} & \textbf{9.30} & \textbf{22.10} & \textbf{2.18} & \textbf{1.52} & \textbf{2.02} & \underline{4.14} & \textbf{22.10} \\
\bottomrule
\end{tabularx}
\vspace{-0.6cm}
\end{table}

\noindent\textbf{Effectiveness under Common Countermeasures:} 
To further evaluate the effectiveness of IRP under defensive scenarios, we test IRP and the baselines on representative defense methods, including Adversarial Training (AT) \cite{goodfellow2014explaining}, Image Shortcut Squeezing (ISS) \cite{liu2023image}, and common data augmentation techniques such as Cutout, CutMix, and Mixup. Following \cite{fu2022robust}, we employ the $L_{\infty}$ norm for AT and set the norm bound to 4/255. 10-step PGD with a step size of 0.6/255 is used for AT. As suggested by \cite{liu2023image}, we employ a combination of grayscale and JPEG with the default JPEG quality factor (JPEG-10) for ISS to ensure global effectiveness against all DAAs. The results on CIFAR-10 and ImageNet-100 are presented in columns 4-8 of Tables \ref{tab:compare_cifar} and \ref{tab:compare_IN}, respectively. The last column displays the maximum clean test accuracy of the corresponding methods across all testing scenarios, representing the worst performance among all the test cases. 

DAA results on CIFAR-10 in Table \ref{tab:compare_cifar} vary across defenses. For instance, OPS shows high effectiveness on AT but is susceptible to ISS and CutMix, while AP performs well against CutMix and Cutout but is less effective against AT and ISS. Only CUDA and IRP demonstrate effectiveness across all countermeasures. However, CUDA permits maximum clean accuracies of 48.58\% across all countermeasures and 66.58\% across all testing cases, which are 16.37\% and 23.24\% higher, respectively, than the 43.24\% maximum accuracy allowed by IRP. 

Similarly, CUDA and IRP emerge as the two most effective methods across all defenses on ImageNet-100 in Table \ref{tab:compare_IN}. However, CUDA permits a maximum clean accuracy of 36.34\%, significantly higher than that of IRP at 22.10\%. The results demonstrate that IRP is the most effective DAA in all testing scenarios.


\noindent\textbf{Evaluation on Other SSL Algorithms:}
To further assess the efficacy of IRP on SSL, we test IRP-poisoned ImageNet-100 on 
three additional SSL algorithms: SimSiam\cite{chen2021exploring}, MoCoV3\cite{chen2021mocov3}, and BYOL\cite{grill2020bootstrap}. As shown in Table \ref{ssl_acc}, IRP lowers the test accuracies of all SSL algorithms to around 10\%. Even in its the worst case performance with BYOL, IRP reduces clean accuracy to 11.6\%. 
\begin{table}[]
\centering
  \caption{Top-1 accuracies (\%) of ResNet-18 trained with different SSL algorithms on ImageNet-100.}
  \label{ssl_acc}
\begin{tabularx}{0.7\textwidth}{l| *{4}{Y}|c}
\toprule
 & SimCLR & MoCoV3 & SimSiam & BYOL & \multicolumn{1}{c}{Max} \\
 \midrule
Clean & 70.96 & 75.54 & 70.68 & 75.34 & 75.54 \\
IRP & 9.3 & 10.58 & 7.82 & 11.6 & 11.6 \\
\bottomrule
\end{tabularx}
\end{table}

\subsection{Additional Training Scenarios and Datasets}

\textbf{Different Poison Percentages:}
In this section, we first evaluate IRP on a more challenging and realistic learning scenario, where only a portion of the data is shielded by IRP. We randomly select a portion of the training data from the entire training dataset of CIFAR-10 and apply IRP to the selected subset. The poisoned data is denoted as $D_p$ and the clean data in the training dataset that has not been used for generating poison data is denoted as $D_c$. Then, we conduct standard supervised training with ResNet-18 on the mixed data ($D_{p+c}$) consisting of the $D_p$ and $D_c$, as well as on the clean data subset $D_c$. The results in Table \ref{tab:partial_poison} show that the effectiveness drops quickly when the data are not 100\% poisoned. This phenomenon applies for all other DAAs as well. Models trained only on $D_c$ demonstrate a similar level of performance as the models trained with $D_{p+c}$. This suggests that the high performance for <100\% poisoning is not due to a failure of the DAAs. Finally, works on DAAs for artwork protection \cite{shan2023glazeprotectingartistsstyle, shan2024nightshadepromptspecificpoisoningattacks} have noted similar trends but find that individuals can still protect their own data via DAA even if the majority of artworks are unprotected. That is, it's still possible to make subsets of the training data unlearnable.

\begin{table}[h]
\caption{The top-1 clean test accuracies (\%) of ResNet-18 trained with partial poisoned data on CIFAR-10. The percentage in the first row indicates the percentage of poison data, and $D_c$ indicates the remaining cleaning data. 0\% and 100\% signify training scenarios where the entire dataset remains either clean or fully poisoned, respectively.}
\label{tab:partial_poison}
\centering
\resizebox{\textwidth}{!}{
\begin{tabularx}{\textwidth}{l|Y| *{2}{Y}|*{2}{Y}|*{2}{Y}|*{2}{Y}|Y}
\toprule
  &   \multirow{2}{*}{0\%} & \multicolumn{2}{c|}{20\%} & \multicolumn{2}{c|}{40\%} & \multicolumn{2}{c|}{60\%} & \multicolumn{2}{c|}{80\%} & \multirow{2}{*}{100\%} \\
  &  & \multicolumn{1}{Y}{$D_{p+c}$} & \multicolumn{1}{Y|}{$D_c$} & \multicolumn{1}{Y}{$D_{p+c}$} & \multicolumn{1}{Y|}{$D_c$} &\multicolumn{1}{Y}{$D_{p+c}$} & \multicolumn{1}{Y|}{$D_c$}& \multicolumn{1}{Y}{$D_{p+c}$} & \multicolumn{1}{Y|}{$D_c$} &  \\
 \midrule
EM & \multirow{8}{*}{93.86} & 94.02 & \multirow{8}{*}{92.38} & 92.72 & \multirow{8}{*}{91.74} & 90.71 & \multirow{8}{*}{87.84} & 85.75 & \multirow{8}{*}{77.36} & 32.53 \\
REM &  & 92.85 &  & 92.15 &  & 90.05 &  & 86.29 &  & 18.91 \\
AP &  & 92.57 &  & \textbf{91.16} &  & 90.72 &  & 88.11 &  & 12.45 \\
LSP &  & 92.95 &  & 92.40 &  & 90.19 &  & 85.76 &  & 21.47 \\
OPS &  & 93.40 &  & 92.25 &  & 90.63 &  & 86.68 &  & 27.24 \\
AR &  & \textbf{92.98} &  & 92.66 &  & 90.94 &  & 87.93 &  & 10.95 \\
CUDA &  & 93.27 &  & 92.27 &  & 90.51 &  & 86.90 &  & 21.89 \\
IRP &  & 93.07 &  & 92.56 &  & \textbf{89.89} &  & \textbf{85.70} &  & \textbf{10.39}\\
\bottomrule
\end{tabularx}
}
\end{table}

\begin{table}[]
\centering
\caption{The top-1 clean test accuracies (\%) of different models trained on data poisoned by IRP and CUDA with AT.}
\label{tab:diff_net}
\resizebox{0.85\textwidth}{!}{
\begin{tabularx}{0.9\textwidth}{l|*{7}{Y}|Y}
\toprule
Dataset &  & RN18 & RN34 & VGG19 & DN121 & MNv2 & ViT-S & Max\\
\midrule
\multirow{3}{*}{CIFAR10} & Clean & 89.57 & 89.32 & 78.57 & 79.97 & 75.94 & 47.84  & 89.57 \\
 & CUDA & 48.58 & 43.63 & 42.02 & 49.44 & 26.33 &  44.93 & 49.44 \\
 & IRP & \textbf{32.21} & \textbf{17.37} & \textbf{19.06} & \textbf{17.87} & \textbf{14.13} &  \textbf{35.39}  & \textbf{35.39}  \\
 \midrule
\multirow{3}{*}{CIFAR100} & Clean & 65.30 & 66.43 & 48.50 & 58.92 & 47.00 & 27.19 & 66.43 \\
 & CUDA & 37.21 & 37.95 & 27.44 & \textbf{23.07} & \textbf{16.00} & 24.88 & 37.95 \\
 & IRP & \textbf{27.13} & \textbf{28.85} & \textbf{22.50} & 25.24 & 20.70 & \textbf{20.82} & \textbf{28.85} \\
 \midrule
\multirow{3}{*}{STL10} & Clean & 71.94 & 69.73 & 64.35 & 72.41 & 58.93 & 43.86 & 72.41 \\
 & CUDA & 44.70 & 44.35 & 45.99 & 39.39 & 44.49 & 42.14 & 45.99 \\
 & IRP & \textbf{34.10} & \textbf{34.33} & \textbf{33.89} & \textbf{33.19} & \textbf{33.81} & \textbf{35.58} & \textbf{35.58}\\
\bottomrule
\end{tabularx}
}
\end{table}

\noindent\textbf{Effectiveness on Different Models:} 
The previous evaluations are all conducted on ResNet-18 (RN18). To evaluate the effectiveness of IRP on different architectures, we evaluate ResNet-34 (RN34), VGG19, DenseNet121 (DN121), MobileNetv2 (MNv2), and ViT-S on IRP-poisoned CIFAR-10. All models are trained with AT ($L_{\infty}$ constraint of $4/255$, as before), which emerged as the strongest defense against different DAAs from our experiments in Section \ref{baslines}. We include CUDA in this evaluation as a reference. On all the examined networks, IRP outperforms CUDA by a significant margin. In their respective worst cases, CUDA permits a test accuracy of 49.44\%, while IRP allows a test accuracy of 35.39\%. Table \ref{tab:diff_net} shows that even under AT defense, IRP is still effective on all the examined architectures.

\noindent\textbf{Evaluation on Different Datasets:}
We extend our evaluation to the CIFAR-100 and STL-10 datasets. The same set of models and AT settings are utilized. As shown in Table \ref{tab:diff_net}, both IRP and CUDA demonstrate effectiveness on CIFAR-100 and STL-10 across all the network architectures under AT. IRP consistently outperforms CUDA by a large margin on most networks for both CIFAR-100 and STL-10. In the two cases where IRP is weaker than CUDA, the clean test accuracies of IRP are still within 5\% of those of CUDA. The maximum clean test accuracies of IRP on CIFAR-100 and STL-10 are 28.85\% and 35.58\%, respectively. Both of these values are significantly lower than those of CUDA, which are 37.95\% and 45.99\% on CIFAR-100 and STL-10, respectively. This experiment demonstrates that IRP is universally effective across networks and across different datasets.

\subsection{Image Quality Comparison against CUDA}


We compare the quality of ImageNet-100 images poisoned by IRP and CUDA. We quantify image quality with three common full-reference quality indices, including LPIPS \cite{zhang2018unreasonable}, SSIM \cite{wang2004image}, and MS-SSIM \cite{wang2003multiscale}, and two no-reference quality indices, including CLIP-IQA \cite{wang2023exploring} and BRISQUE \cite{mittal2012no}. The original image is used as the reference for LPIPS, SSIM, and MS-SSIM. Table \ref{tab:im_quality} presents the evaluation results of all quality indices. Both Table \ref{tab:im_quality} and Fig. \ref{cuda_irp} demonstrate that IRP generates poisoned images with better image quality than CUDA.

\begin{figure}
  \begin{minipage}[b]{.4\linewidth}
    \centering
    \begin{tabularx}{1.0\textwidth}{lccc}
        \toprule
        Metric & Clean & CUDA & IRP \\
        \midrule
        LPIPS $\downarrow$ & - & 0.272 & 0.142 \\
        SSIM $\uparrow$& - & 0.561 & 0.775 \\
        MS-SSIM$\uparrow$ & - & 0.831 & 0.903 \\
        CLIP-IQA$\uparrow$& 0.787 & 0.528 & 0.706 \\
        BRISQUE$\downarrow$ & 16.625 & 29.570 & 17.318 \\
        \bottomrule
    \end{tabularx}
    \captionof{table}{ImageNet-100 image quality under different quality indices.}
    \label{tab:im_quality}
  \end{minipage}
  \hfill
  \begin{minipage}[b]{.5\linewidth}
    \centering
    \includegraphics[width=1.0\textwidth]{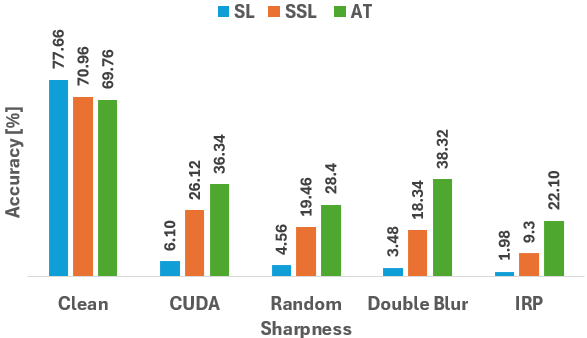}
    \captionof{figure}{The Top-1 accuracy of different setting on ImageNet-100.}
    \label{ablation}
  \end{minipage}
\end{figure}


\subsection{Filter Ablation}

The IRP filter $P_c=R_c \star A_c$ is a kernel with a larger size of $(2\kappa+1)\times(2\kappa+1)$, where $R_c$ and $A_c$ are both filters of size $\kappa\times \kappa$. To verify that the high effectiveness of IRP is not caused by the bigger kernel size, we generate filters of the same size as the IRP under two different filter settings. For `Random Sharpness', we generate a random sharpness filter $A_{c1}$ of size $\kappa \times \kappa$ for each class. Within each $A_{c1}$, one random parameter out of the $\kappa \times \kappa$ parameters is set to 1, while the remaining parameters are randomly initialized from a uniform distribution with support [$-p_s$, 0]. After testing several values for $p_s$, we find that 0.01 is suitable for generating poisoned images with good visual quality. We then poison class $c$ images by $X_c \star R_{c} \star A_{c1}$. For `Double Blur', we generate another random CUDA filter $R_{c1}$ of size $\kappa \times \kappa$ for each class, and then poison class c images by $X_c \star R_{c} \star R_{c1}$. 

We compare the effectiveness of the two additional filter settings with CUDA and IRP on SL, SSL, and AT for ImageNet-100. The results in Fig. \ref{ablation} show that Random Sharpness and Double Blur exhibit slightly better performance than CUDA in SSL. However, the effectiveness of Random Sharpness and Double Blur cannot match that of IRP. Additionally, we provide examples of images generated by Double Blur and Random Sharpness in Appendix \ref{subsec:DoubleBlur_RandomSharp}. These samples highlight their inferior quality compared to those generated by IRP.



\section{Conclusion}

In this study, we conduct a theoretical analysis of CUDA, uncovering the sub-optimal gradients it introduces and elucidating the strategy it employs to induce class-wise bias for data poisoning. Building on these insights, we introduce IRP, which demonstrates high effectiveness in both SL and SSL as well as under state-of-the-art defense scenarios. IRP also maintains high image quality, rendering it suitable for data protection in real-world applications.


\clearpage
\section*{Acknowledgements}

This research is supported by the National Research Foundation, Singapore and Infocomm Media Development Authority under its Trust Tech Funding Initiative and Strategic Capability Research Centres Funding Initiative. Any opinions, findings and conclusions or recommendations expressed in this material are those of the author(s) and do not reflect the views of National Research Foundation, Singapore and Infocomm Media Development Authority.

Mingzhi Lyu and Fan Wang are supported by ROSE @ NTU, Interdisciplinary Graduate Programme, Nanyang Technological University, Singapore.

\bibliographystyle{splncs04}
\bibliography{main}

\clearpage
\input{Appendix}

\end{document}

%% file: Appendix.tex
\clearpage
\appendix{}

\section{Proofs}

\subsection{Proof of Theorem 1}
\label{subsec:Proof_Thm1}

\begin{lemma}
    \label{lemma:sq_uniform}
    Given a random variable, $X$ follows a uniform distribution with support $U(0, p_u)$. The p.d.f. of $X^2$ is $\frac{1}{2p_u \sqrt{x}}$ and  $E(X^2) = \frac{p_u^2}{3}$
\end{lemma}

\begin{proof}
Considering
\begin{equation}
    \Pr(X^2 < x) = \Pr(X < \sqrt{x}) = \int_{0}^{\sqrt{x}} \frac{1}{p_u} \, dx = \frac{1}{p_u} \sqrt{x},
\end{equation}
which is the c.d.f. of $X^2$. To obtain the p.d.f. of $X^2$, we compute $\frac{d}{dx} \left( \frac{1}{p_u} \sqrt{x} \right) = \frac{1}{2p_u \sqrt{x}}
$. Note that the support of $X^2$ is $[0, p_u^2]$ not $[0, p_u]$. To obtain
\begin{equation}
    E(X^2) = \int_{0}^{p_u^2} \frac{1}{2p_u \sqrt{x}} \, x \, dx = \frac{1}{2p_u} \frac{2}{3} x^{3/2} \bigg|_{0}^{p_u^2} = \frac{1}{3p_u} p_u^3 = \frac{p_u^2}{3}
\end{equation}
q.e.d.
\end{proof}

\setcounter{theorem}{0}

\begin{theorem}
\label{theorem1}
    Given two different filters, $R_i$ and $R_c$, generated by CUDA,  if the condition $1 > (\kappa^2 - 2) p_u^2 + 2p_u$ is satisfied, 
    $R_{c\pi}\star R_{i\pi}$ and $R_{c\pi}\star R_{c\pi}$ have the following properties.
    \begin{enumerate}
        \item The peak of $R_{c\pi}\star R_{c\pi}$ is located at the center.
        \item The peak of $R_{c\pi}\star R_{i\pi}$ occurs at the position where two ones in  $R_{c\pi}$ and  $R_{i\pi}$ appear at the same position in $R_{c\pi}\star R_{i\pi}$.
        \item The probability of peak of $R_{c\pi}\star R_{i\pi}$  at the center is $\frac{1}{\kappa^2}$.
        \item  	The expected peak of $R_{c\pi}\star R_{c\pi}$  is higher than the expected peak of $R_{c\pi}\star R_{i\pi}$.
    \end{enumerate}
\end{theorem}

\begin{proof}
 
 Property 1: The center of $R_{c\pi} \star R_{c\pi}$ is at $(\kappa,\kappa)$, where two filters completely overlap, and its minimum value is larger than 1. The maximum value at other positions is less than $(\kappa^2 - \kappa - 2) p_u^2 + 2p_u$. Using the assumption, we can deduce that $R_{c\pi} \star R_{c\pi} (\kappa,\kappa) > R_{c\pi} \star R_{c\pi} (i,j)$, for all $(i,j) \neq (\kappa,\kappa)$.

Property 2: When the one in $R_{c\pi}$ and the one in $R_{j\pi}$ appear at the same position in $R_{c\pi} \star R_{j\pi}$, the minimum value of $R_{c\pi} \star R_{j\pi}$ at the location is larger than one. The maximum value of other positions is less than $(\kappa^2 - 2) p_u^2 + 2p_u$. Based on the assumption, we achieve property 2.

Property 3: Based on property 2, the peak of $R_{c\pi} \star R_{j\pi}$ is at the position where the one in $R_{c\pi}$ and the one in $R_{j\pi}$ appear at the same position in $R_{c\pi} \star R_{j\pi}$. If the peak is at the center of $R_{c\pi} \star R_{j\pi}$, it means that the one in $R_{c\pi}$  and the one in $R_{j\pi}$ must be at the position. The probability that they are in the position is $\frac{1}{\kappa^2}$.

 Property 4: Utilizing Lemma \ref{lemma:sq_uniform} and the preceding properties, we can derive the expected peak value of $R_{c\pi} \star R_{c\pi}$:
 \begin{equation}
     E(R_{c\pi} \star R_{c\pi} (\kappa,\kappa)) = 1 + (\kappa^2 - 1) \left(\frac{p_u^2}{3}\right).
 \end{equation}
When two ones are at the same position in $R_{c\pi}$ and $R_{i\pi}$, $R_{c\pi} \star R_{i\pi} (\kappa,\kappa)$ achieves its maximum value. Thus the expected maximum peak value of $R_{c\pi} \star R_{i\pi}$ is
\begin{equation}
    E(R_{c\pi} \star R_{i\pi} (\kappa,\kappa)) = 1 + (\kappa^2 - 1) \left(\frac{p_u^2}{4}\right).
\end{equation}
\end{proof}

\subsection{Proof of Theorem 2}
\label{subsec:Proof_Thm2}

\begin{theorem}
\label{theorem2}
Given a CUDA filter size being larger than 1 and $x_{cij}$ being the vector form of the patch in $X_{ci}$ involved in producing $\eta_{cij}$, i.e., $\eta_{cij} = \wp x_{cij}$
, where $\wp$ is a $\kappa^2 \times (2\kappa-1)^2$ matrix,
if the column vectors of $\wp$, excluding the one corresponding to $y_{cij}$, span $\mathbb{R}^{\kappa^2}$, then, $\exists x_{cij} \in \mathbb{R}^{(2\kappa-1)^2}$ such that $y_{cij} \neq a_c^T \wp x_{cij}$.
\end{theorem}

\begin{proof}
    Let the patch in $X_{ci}$ involved in computing the patch $(X_{ci} \star R_c)_j$ be $(X_{ci})_j$. Since the patch size of $(X_{ci} \star R_c)_j$ and the size of $R_c$ are both $\kappa \times \kappa$, $((X_{ci})_j)$ has the size of $(2\kappa - 1) \times (2\kappa - 1)$. When $\kappa > 1$, $(X_{ci})_j$ is larger than $(X_{ci} \star R_c)_j$. Reshaping $(X_{ci})_j$ and $(X_{ci} \star R_c)_j$ as vectors, $x_{cij}$ and $\eta_{cij}$, their relationship can be represented as a linear equation, i.e., $\eta_{cij} = \wp x_{cij}$, where $\wp$ is a $\kappa^2 \times (2\kappa - 1)^2$ matrix. The matrix $\wp$ is constructed by $R_c$. To achieve perfect reconstruction $y_{cij} = (\alpha_c^T \wp) x_{cij}$ for all $x_{cij} \in \mathbb{R}^{(2\kappa - 1)^2}$, $\alpha_c^T \wp$ must be $[0 \cdots 0 \, 1 \, 0 \cdots 0]$, where the location of the one corresponds to $y_{cij}$ in $(X_{ci})_j$. Since the column vectors of $\wp$, excluding the one corresponding to $y_{cij}$, span $\mathbb{R}^{\kappa^2}$, $\alpha_c$ cannot be orthogonal to all these column vectors. Thus, $\alpha_c^T\wp = [0 \cdots 0 \, 1 \, 0 \cdots 0]$ cannot be obtained.
\end{proof}

\subsection{Extending CUDA to CIFAR-10}
\label{subsec:Proof_CIFAR10}

Since CUDA uses $p_u=0.3$ and $\kappa=3$ for CIFAR-10, it fulfils the condition $1 > \frac{1}{4}(\kappa^2 - 2)p_u^2 +  p_u$. According to Theorem \ref{theorem3}, CUDA filters exhibit similar properties shown in Theorem \ref{theorem1} on CIFAR-10.
\begin{theorem}
\label{theorem3}
    Given two different filters, $R_i$ and $R_c$ generated by CUDA,  if the condition $1 > \frac{1}{4}(\kappa^2 - 2) p_u^2 + p_u$ is satisfied, 
    $R_{c\pi}\star R_{i\pi}$ and $R_{c\pi}\star R_{c\pi}$ have the following properties.
    \begin{enumerate}
        \item The expected peak of $R_{c\pi}\star R_{c\pi}$ is located at the center.
        \item The expected peak of $R_{c\pi}\star R_{i\pi}$ occurs at the position where two ones in  $R_{c\pi}$ and  $R_{i\pi}$ appear at the same position in $R_{c\pi}\star R_{i\pi}$.
        \item The probability of the expected peak of $R_{c\pi}\star R_{i\pi}$ at the centre is $\frac{1}{k^2}$.
        \item  	The expected peak of $R_{c\pi}\star R_{c\pi}$ is higher than the expected peak of $R_{c\pi}\star R_{i\pi}$.
    \end{enumerate}
\end{theorem}
\begin{proof}
Property 1: The center of $R_{c\pi} \star R_{c\pi}$ is at $(\kappa, \kappa)$, where two filters completely overlap, and its expected value is
\begin{equation}
    E(R_{c\pi} \star R_{c\pi} (\kappa, \kappa)) = 1 + \frac{1}{3} (\kappa^2 - 1) p_u^2.
\end{equation}
Let $\omega$ be the number of pixels overlapped when computing $E(R_{c\pi} \star R_{c\pi} (i,j))$, where $(i,j) \neq (\kappa,\kappa)$. We have the following cases:

\noindent Case 1: If both ones are not inside the overlapping region, 
\begin{equation*}
    E(R_{c\pi} \star R_{c\pi} (i,j)) = \frac{1}{4} \omega p_u^2
\end{equation*}
Case 2: If both ones are inside the overlapping region but at different locations, 
\begin{equation*}
    E(R_{c\pi} \star R_{c\pi} (i,j)) = \rho_u + \frac{1}{4}(\omega - 2)p_u^2
\end{equation*}
Case 3: If both ones are inside the overlapping region and at the same location, 
\begin{equation*}
    E(R_{c\pi} \star R_{c\pi} (i,j)) = 1 + \frac{1}{4}(\omega - 1)p_u^2
\end{equation*}
Case 4: If only one is inside the overlapping region, 
\begin{equation*}
    E(R_{c\pi} \star R_{c\pi} (i,j)) = \frac{1}{2} \rho_u + \frac{1}{4}(\omega - 1)p_u^2
\end{equation*}
Among the four cases, the third case is the largest one, since $\rho_u < 1$. Clearly $E(R_{c\pi} \star R_{c\pi} (\kappa,\kappa))$ is larger than $E(R_{c\pi} \star R_{c\pi} (i,j))$ in the third case. Thus, the expected peak of $E(R_{c\pi} \star R_{c\pi} (\kappa,\kappa))$ is at the center.

{Property 2}: When two ones overlap, the minimum value of $E(R_{c\pi} \star R_{i\pi} (i,j))$ is larger than 1. The expected maximum value of $E(R_{c\pi} \star R_{i\pi} (i,j))$ when two ones do not overlap is
\begin{equation*}
    E(R_{c\pi} \star R_{i\pi} (i,j)) = \rho_u + \frac{1}{4}(\kappa^2 - 2)p_u^2.
\end{equation*}
Using the assumption, the expected peak appears at the position when two ones overlap. 

{Property 3: }Based on property 2, the expected peak appears at the position when two ones are at the same position. Thus, the probability of the expected peak appearing at the center is $\frac{1}{\kappa^2}$

{Property 4}: The proof is the same as the proof of property 4 in Theorem \ref{theorem1}.

\end{proof}

\section{Experiment Setup}

\subsection{Training details}
\label{subsec:Training_Details}

\begin{table}[htbp]
  \centering
  \caption{Training settings for  SL and SSL. For cells with two values, the left value is for CIFAR-10 and the right value is for ImageNet-100. `Pretrain' is the standard augmentation utilized in SSL pretraining, while `LinProbe' is the standard augmentation employed in SSL linear probing.}
  \label{training}
  \begin{tabular}{lccccc}
    \toprule
    Setting  & SSL Pretrain & SSL Linear Probing & SL \\
    \midrule
    Augmentations  & Pretrain & LinProbe & LinProbe \\
    Epochs & 400 | 200 & 100 & 100 \\
    Batch Size  & 512 | 128 & 512 & 256 \\
    LR x BatchSize/256  & 1.0 | 0.25 & 1.0 | 10.0 & 0.1 \\
    LR Warmup & 10 & 0 & 0 \\
    LR Decay  & Cosine & Step (0.2x at 60,75,90) & Cosine \\
    Optimizer & SGD & SGD & SGD \\
    Momentum  & 0.9 & 0.9 & 0.9 \\
    Weight Decay  & 1e-4 & 0 & 5e-4 \\
    \bottomrule
  \end{tabular}
  \label{tab:Train_Settings}
\end{table}

The training details are provided in Table \ref{training}. However, there are two exceptions: for VGG-19 and ViT models, the learning rate in SL is set to 0.01, and the training epoch for ViT is adjusted to 200. For the SL training, we utilize the standard linear probing augmentation of SSL as it has demonstrated an enhancement in clean test accuracy. This implies that the linear probing augmentation acts as an effective defense against certain DAAs. Specifically, in ImageNet training, the data augmentation comprises RandomResizedCrop with an expected output size of 224 along with RandomHorizontalFlip. For CIFAR-10 training, RandomCrop with an output size of 32 and padding set to 4 is employed, alongside RandomHorizontalFlip. For SSL pretraining and linear probing, we follow the standard augmentations used in \cite{chen2020simple, chen2021exploring, chen2021mocov3, grill2020bootstrap}. 

\section{Additional Results}

\subsection{Images with Double Blur and Random Sharpness}
\label{subsec:DoubleBlur_RandomSharp}

The examples presented in Figure \ref{fig:pioson_example} depict outputs generated through Double Blur, Random Sharpness, and IRP. The images produced by Double Blur exhibit the highest level of blurriness. The degree of blurriness in the images generated by Random Sharpness varies, as the sharpness filters are randomly generated class-wise. Conversely, the images produced by IRP are the clearest, owing to the optimization process customized for the blurriness filter in IRP.

\begin{figure}
    \centering
    \includegraphics[width=\textwidth]{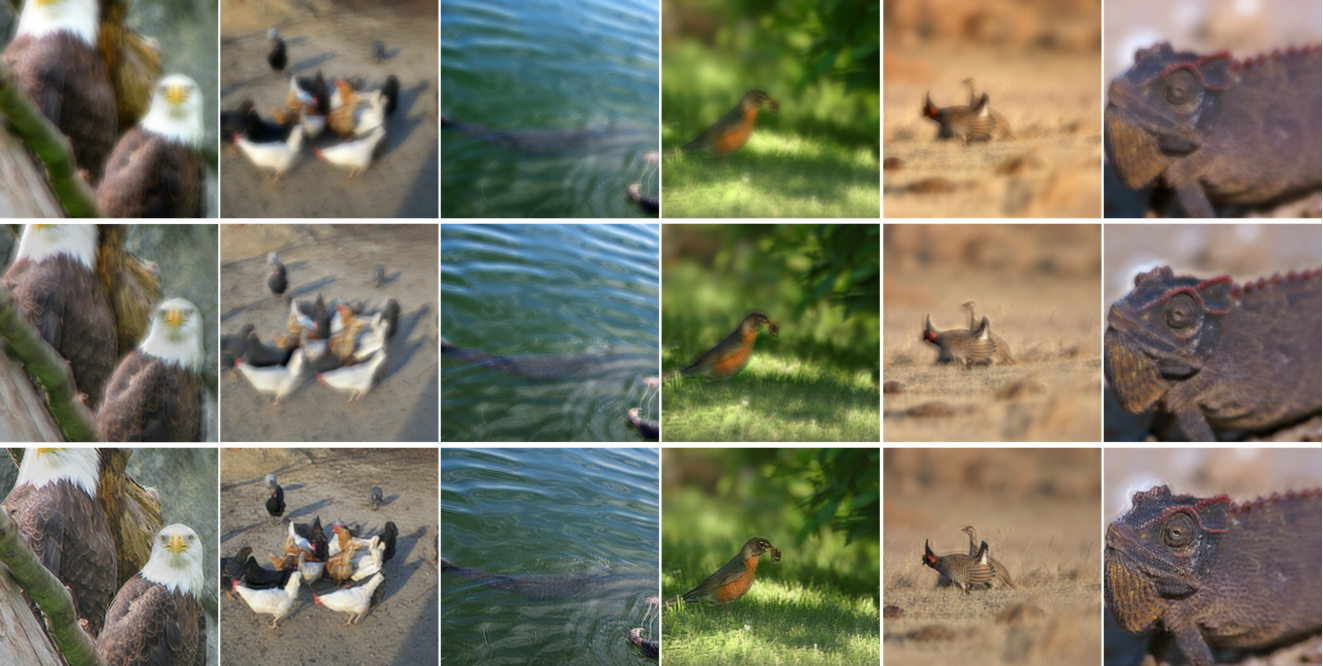}
    \caption{Example images generated by Double Blur (the first row), Random Sharpness (the second row), and IRP (the third row) on ImageNet-100.}
    \label{fig:pioson_example}
\end{figure}

\subsection{Additional Defenses}

We further evaluate IRP on three recently invented poison defenses. AVATAR \cite{dolatabadi2024devilsadvocateshatteringillusion} adds Gaussian noise to poisoned images and `purifies' them with a pretrained diffusion model. UEraser \cite{qin2023learningunlearnableadversarialaugmentations} applies multiple augmentations to poisoned images and selects the one that maximizes loss. COIN \cite{wang2024corruptingconvolutionbasedunlearnabledatasets} applies random pixel interpolations to reduce the impact of filter-based poisons. 

Figure \ref{tab:DAA_Defenses} shows that AVATAR and COIN are useful defenses against both CUDA and IRP poisons. Under COIN purification, CUDA-poisoned CIFAR-10 permits a clean test accuracy of 71.90\%. AVATAR raises the max accuracy of IRP-poisoned CIFAR-10 to 54.78\%, though this value is still severely low compared to training on clean data. UEraser is less effective as a defense than standard methods like SSL and AT. Across all defenses, IRP is a more potent DAA than CUDA.

\begin{table}[h]
  \centering
  \caption{Top-1 clean test accuracies (\%) of ResNet-18 trained with data poisoned by various DAAs on CIFAR-10. The ``Max'' column represents the highest clean test accuracy, which corresponds to the worst performance of each DAA under different training scenarios. Bold indicates the best performance.}
  \begin{tabular}{l|ccc|ccc|c}
    \toprule
    DAA & SL & SSL & AT & AVATAR & UEraser & COIN & \multicolumn{1}{c}{Max} \\
    \midrule
    Clean & 93.86 & 90.55 & 89.57 & - & - & - & 93.86 \\
    \midrule
    CUDA & 21.89 & 66.58 & 48.58 & 55.93 & 40.56 & 71.90 & 71.90 \\
    IRP & \textbf{10.39} & \textbf{43.24} & \textbf{32.21} & \textbf{54.78} & \textbf{26.55} & \textbf{53.32} & \textbf{54.78} \\
    \bottomrule
  \end{tabular}
  \label{tab:DAA_Defenses}
\end{table}